\documentclass{article} 
\usepackage{iclr2016_workshop,times}
\usepackage{hyperref}
\usepackage{url}
\usepackage{amsmath}
\usepackage{amsthm}
\newtheorem{prop}{Proposition}
\newtheorem{lemma}{Lemma}
\usepackage{amssymb}
\usepackage{graphicx}
\usepackage{caption}
\usepackage{subcaption}
\newcommand{\cost}{\textrm{cost}}
\newcommand{\var}{\textrm{Var}}
\newcommand{\Var}[1]{\textrm{Var}\left( #1\right)}
\newcommand{\EE}[1]{\mathbb{E}\left[ #1 \right]}
\newcommand{\cov}{\textrm{Cov}}
\newcommand{\E}{\textrm{E}}
\newcommand{\real}{\mathbb{R}}
\title{Revise Saturated Activation Functions}

\author{Bing Xu, Ruitong Huang  \\
Department of Computing Science\\
University of Alberta\\
\texttt{\{antinucleon,rtonghuang\}@gmail.com} \\
\And
Mu Li \\
Computer Science Department \\
Carnegie Mellon University \\
\texttt{muli@cs.cmu.edu } \\
}

\newcommand{\cl}[1]{{(#1)}}
\begin{document}

\maketitle

\begin{abstract}
  In this paper, we revise two commonly used saturated functions, the
  logistic sigmoid and the hyperbolic tangent (tanh).
  We point out that, besides the well-known non-zero centered property, slope of the activation function near the origin is another possible
  reason making training deep networks with the logistic function
  difficult to train. We demonstrate that, with proper rescaling, the logistic sigmoid
  achieves comparable results with tanh.
  Then following the same argument, we improve tahn by penalizing in the negative part. 
  We show
  that ``penalized tanh'' is comparable and even outperforms the state-of-the-art
  non-saturated functions including ReLU and leaky ReLU on deep convolution
  neural networks.

  Our results contradict to the conclusion of previous works that the saturation
  property causes the slow convergence. It suggests further investigation is necessary to
  better understand activation functions in deep architectures.
\end{abstract}

\section{Introduction}
Activation functions play an important role in artificial neural networks in that they help bring non-linearity to the networks. 
Different activation functions can significantly affect the performance of a neural network, and therefore how to choose a good activation function has attracted lots of studies in the literature. 
One of the publicly accepted arguments is that a saturated activation function may cause the gradient vanishing (and/or explosion), and thus is less preferred.
In particular, it has been reported that the backpropagated gradient of a network which uses the logistic sigmoid $f(x) = 1/(1+e^{-x})$ as its activation function may vanish or explode quickly.
Currently there is no success, to our best knowledge, in training a deep neural network with this activation function without Batch Normalization\citep{ioffe2015batch}.
However, different from logistic sigmoid, the performances of networks with tanh, which is also saturated, are much more stable.
For example, a deep convolution neural network using tanh is able to reach a local optimality with careful Layer-sequential unit-variance weight initialization \citep{mishkin2015all}. 
It then raise a question what makes these two functions so different, despite that they are both saturated， by which our investigations in the paper are driven.

We start with verifying the assumptions that are required in Xavier initialization \citep{glorot2010understanding} for a general activation function in its linear regime, based on which we discuss the failure of logistic sigmoid and propose two methods to overcome the training problem of the deep Sigmoid networks.
Our analysis suggests that besides the well-known non-zero centered property, slope and the offset of the activation function near the origin is another possible reason causing the vanishing (and/or explosion) of the gradient.


One well accepted explanation about the empirical success of ReLU is about its non-saturated property, compared to other saturated functions.
Similarly, we design new activation functions to investigate the essential effects of being leaky, its linear regime, and its saturation property outside the linear regime of an activation function.
In particular, we compare the performance of a new activation function, called leaky tanh, to the performances of ReLU and leaky ReLU.
The new function penalized tanh shares similar property in its linear regime with leaky ReLU, yet different from ReLU or leaky ReLU, it is saturated outside its linear regime.
Our results provide more insights about the effect of different activation functions on the performance of the neural networks, and suggest that further investigation is still needed for better understanding.

All the networks in the paper are trained by using MXNet \citep{chen2015mxnet}.

\section{Why Training Deep Neural Networks is Hard with the Logistic Sigmoid }
We first investigate the behaviors of the activation variance and the gradient variance,
for a general activation function in its linear regime, within the theoretical framework developed in \citep{glorot2010understanding}.
Our analysis is focusing on the initialization stage.
Basis on the analysis, conditions that are required for the activation function to maintain the activation variance and the gradient variance are presented, 
which then leads to the discussion about the difficulty of training a deep neural networks using the logistic sigmoid. 

To simplify the analysis, we assume the following fully connected neural network. 
Let $y^{(l)} \in \mathbb{R}^{n_l}$ be the output of layer $l$ for $l = 1,2,\ldots, L$, and $f$
be the activation function, in a forward pass we have
\begin{align}
x^\cl{l} &= W^\cl{l-1}y^\cl{l-1}+b^\cl{l-1} \\
y^\cl{l} &= f(x^\cl{l}) \\
\end{align}
where the weight $W^{(l-1)} \in \mathbb{R}^{n_{l} \times n_{l-1}}$ and the bias $b^{\cl{l-1}}\in \mathbb{R}^{n_{l}}$.

Assume that we initialized  $b^\cl{l-1}$ to be 0.
Thus the randomness of $y^\cl{l}$ comes from both the output of previous layer $y^{\cl{l-1}}$ and the weight $W^{\cl{l-1}}$ which is randomly initialized.
Further assume that for each $W^{\cl{l}}$, its elements are initialized independently with $\EE{W^\cl{l}} = 0$ with equal variance $\sigma_l^2$.  
Also assume that all the weights $\{W^\cl{l};\, l=0,\ldots,L-1\}$ are mutually independent.

\begin{prop}
\label{prop:var}
Assume that $\frac{\partial \cost}{\partial y^\cl{L}}$ is independent to $W^\cl{t}$, $t = 1,\ldots,L-1$, with $\Var{\frac{\partial \cost}{\partial y^\cl{L}}} = d_LI$ for some constant $d_L$. Also assume that  $\EE{\frac{\partial \cost}{\partial y^\cl{L}}} = 0$.
Let the activation function $f(x) = \alpha x + \beta$, then for any layer $l$,
 \begin{align}
\Var{y^\cl{l}} & = \alpha^2n_{l-1}\sigma_{l-1}^2\left( \Var{y^\cl{l-1}} +\beta^2 I_{n_l}\right). \\
\Var{\frac{\partial \cost}{\partial y^\cl{l-1}}} & = \alpha^2n_l\sigma_{l-1}^2\Var{\frac{\partial \cost}{\partial y^\cl{l}}}.
\end{align}  
\end{prop}

The proof is formally developed in Appendix~\ref{sec:appendix}, following the idea of \citet{glorot2010understanding}.
Assume that the network is initialized using Xavier initialization, then to maintain the activation variance and the gradient variance, namely for $l = 1, \ldots, L$,
\begin{align*}
& \Var{y^{(l)}} = \Var{y^{(l-1)}}
\quad\textrm{and}\quad
\Var{\frac{\partial \textrm{cost}}{\partial y^{(l)}}} =
\Var{\frac{\partial \textrm{cost}}{\partial y^{(l-1)}}};
\\
& n_l\sigma_{l-1}^2 \approxeq 1 
\quad\textrm{and}\quad 
n_{l-1}\sigma_{l-1}^2 \approxeq 1;
\end{align*}
$\alpha$ and $\beta$ must satisfy that
\[
\alpha = 1 \quad \text{and} \quad \beta = 0.
\]

Now consider the Taylor expansions of different activation functions,
\begin{align}
\mbox{sigmoid}(x) & = \frac{1}{2} + \frac{x}{4} - \frac{x^3}{48} + O(x^5)\\
\mbox{tanh}(x) & = 0 +  x - \frac{x^3}{3} + O(x^5)\\
\mbox{relu}(x) & = 0 + x \quad\textrm{for } x \ge 0.
\end{align}
It shows that in their linear regimes, both tanh and relu have the desirable property that $\alpha=1$ and $\beta = 0$. 
If the weight is initialized with zero mean and $1/n$
variance, which is one of the widely used method \cite{glorot2010understanding},
then both the forward output and backward gradient will be in proper range at
least for the first few iterations. 
However, this is not true for logistic sigmoid. 
First, its slope in the linear regime is $1/4$ rather
than 1, then we need to initialize the weight is a 16 times smaller variance to
keep the each layer's gradient variance the same. Second, it has a non-zero
mean, which makes the output variance increase linear with the layer.

One simply way to fix this problem is rescale the logistic sigmoid to match the
first two degree coefficient with both tanh and relu, so the weight
initialization method used for the latter two can be applied to the logistic
sigmoid. In other word, we transform the logistic sigmoid by

\begin{equation}
\textrm{scaled sigmoid}(x) = 4 \times \textrm{sigmoid}(x) - 2 = \frac{4}{1+e^{-x}} - 2
\end{equation}

The scaled logisitc sigmoid is illustrated in Figure~\ref{fig:sigmoid}. As can
bee seen, it is similar to tanh near 0, but the saturation value is two times
larger than tanh.
Our experimental results shows that the scaled sigmoid
 function achievess comparable results with tanh.
 
In the other perspective, the scale factor 4 in scaled sigmoid function is equivalent to scale initialization and learning rate by factor of 4; the bias term -2 in scaled sigmoid function is equivalent to fixed bias after linear transform.  

\begin{figure}[th!]
  \centering
  \includegraphics[width=.5\textwidth]{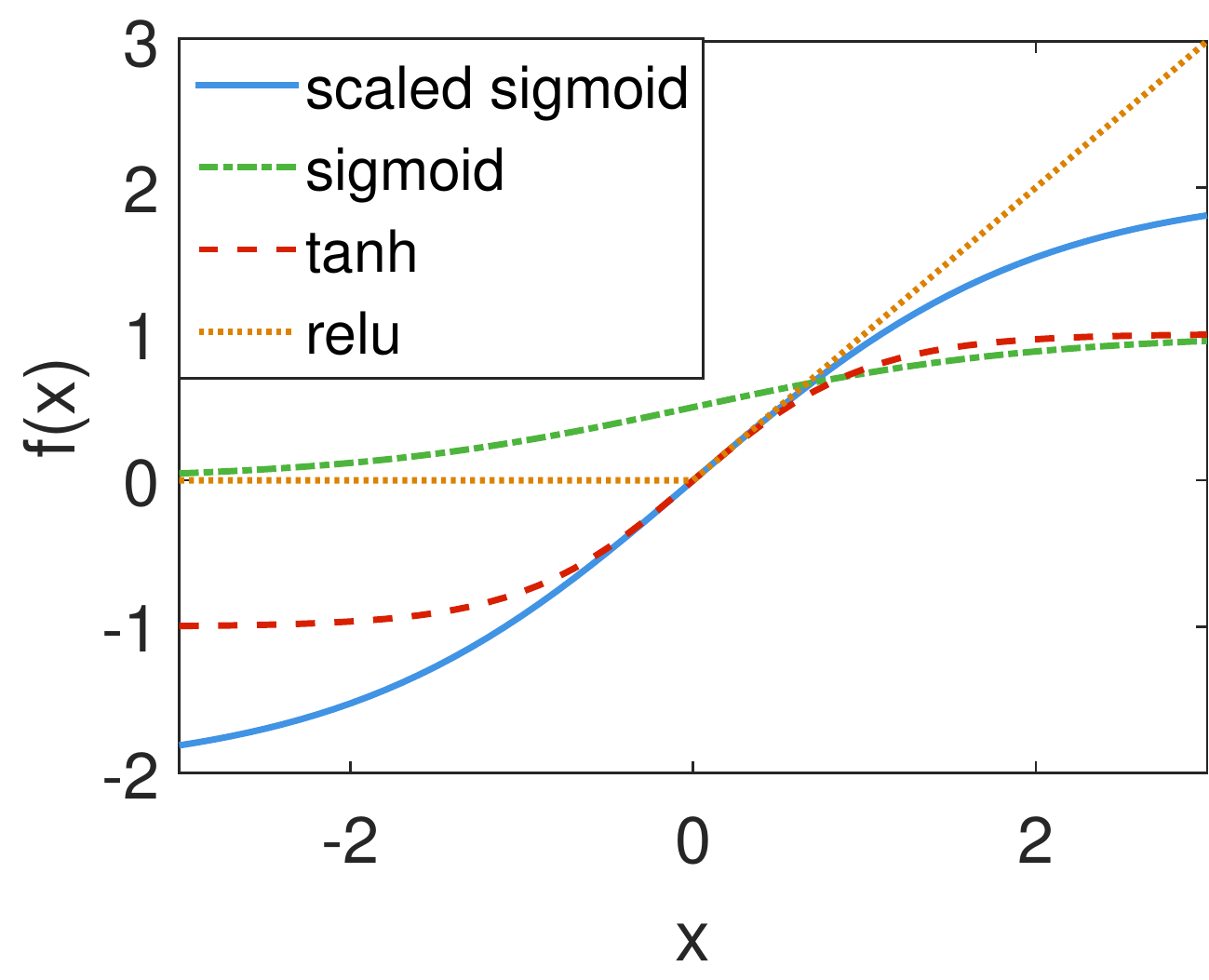}
  \caption{Compare the scaled logistic sigmoid with other activation functions.\label{fig:sigmoid}}
\end{figure}

\section{Penalized Saturated Activation Functions}
Recently, several variants have been proposed in the literature to improve the performance of non-saturated activation functions.
One that we are particularly interested in is the `leaky ReLU', which is as follows.

\begin{equation}
  f(x) =
\begin{cases}
  x & \textrm{if } x > 0\\
  ax & \textrm{otherwise}
\end{cases}，
\end{equation}
where $a\in(0,1)$.
Comparing to the standard ReLU, leaky ReLU gives nonzero gradient for
negative value. 
Although its output is no longer sparse, which is claimed to be the main advantage of ReLU, recent works show that oftentimes leaky ReLU outperforms ReLU \citep{xu2015empirical, clevert2015fast}.

On the other viewpoint, the leaky ReLU can be viewed as an improvement over the
simple identical activation function $f(x) = x$ which penalizes the gradient of
the negative part. Inspired by this observation, we propose to also penalize the
negative part of the saturated activation functions.
In particular, we propose
``penalized tanh'' which takes the form
\begin{equation}
  f(x) =
\begin{cases}
  \textrm{tanh}(x)  & \textrm{if } x > 0\\
  a \cdot \textrm{tanh}(x) & \textrm{otherwise}
\end{cases}
\end{equation}
where $a\in(0,1)$.
Figure~\ref{fig:tanh} compares the penalized tanh with other activation
functions. As can be seen, if the same $a$ is used, the penalized tanh can be
viewed as a saturated version of leaky ReLU. 
These two functions have similar value near 0, since both function share the same Taylor expansion up to the first order.
But different to leaky Relu, penalized tanh saturates to $-a$ and $1$
when moving away from 0.

\begin{figure}[th!]
  \centering
  \includegraphics[width=.5\textwidth]{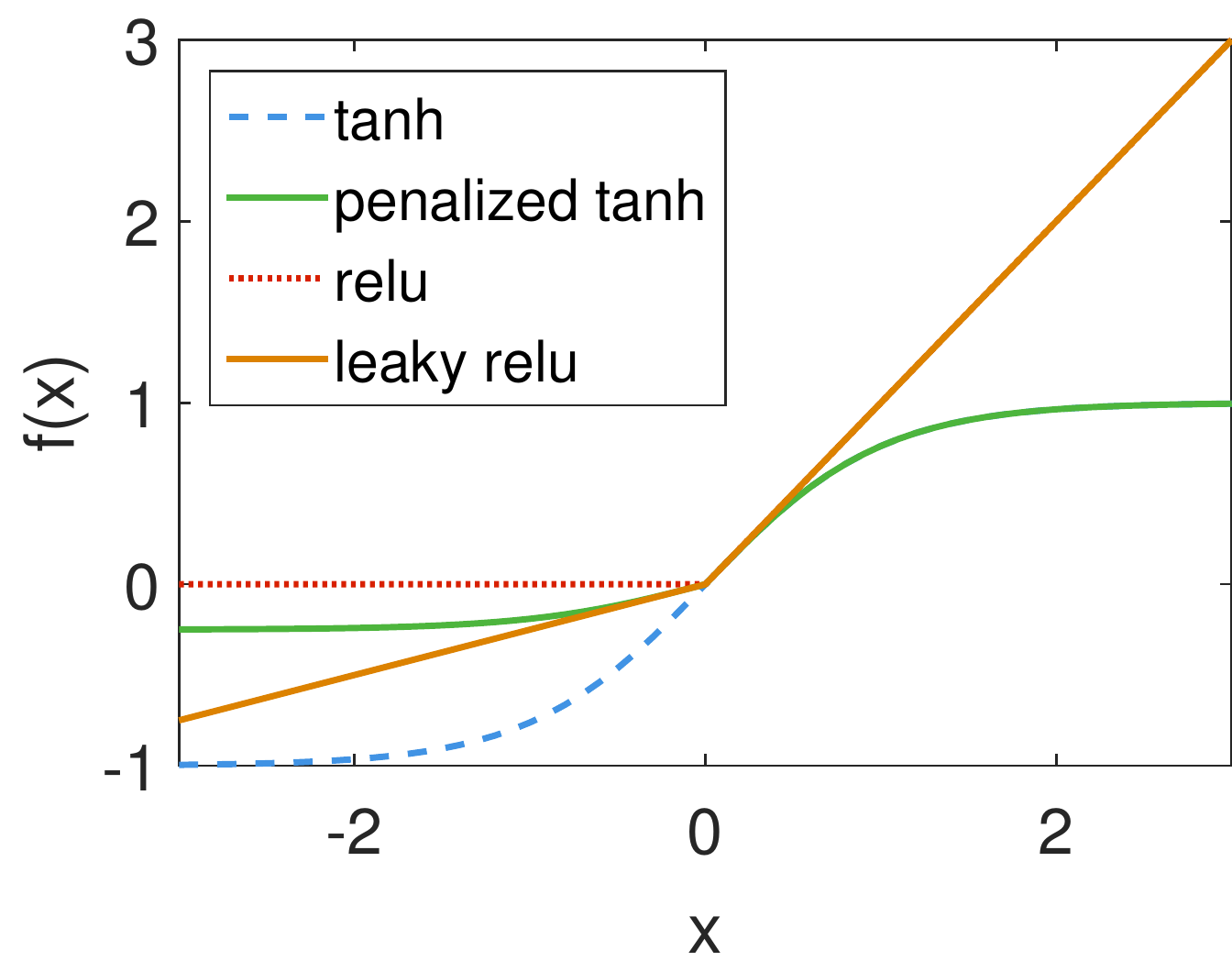}
  \caption{Compare the penalized tanh with other functions. Both penalized tanh
    and leaky ReLU uses $a=1/4$.}
  \label{fig:tanh}
\end{figure}

\section{Experiment}

\begin{table}[htb!]
    \center
   \begin{tabular}{|l|l|l|}
   \hline
   Activation            & Train Accuracy & Test Accuracy \\ \hline
   sigmoid               & diverged               & diverged          \\  
   scaled sigmoid              & 89.39\%        & 59.11\%       \\ 
   tanh                  & 96.94\%        & 61.99\%       \\ 
   ReLU                  & 99.17\%        & 67.91\%       \\ 
   penalized tanh (a = 0.25) & 99.75\%        & 70.43\%       \\ 
   leaky ReLU (a = 0.25) & 99.85\%        & 70.64\%       \\ \hline
   \end{tabular}
   \caption{Vary activation function for inception network on CIFAR-100}
   \label{tab:cifar100}
\end{table}

\begin{figure}[th!]
  \centering
  \begin{subfigure}[b]{0.5\textwidth}
  \includegraphics[width=\textwidth]{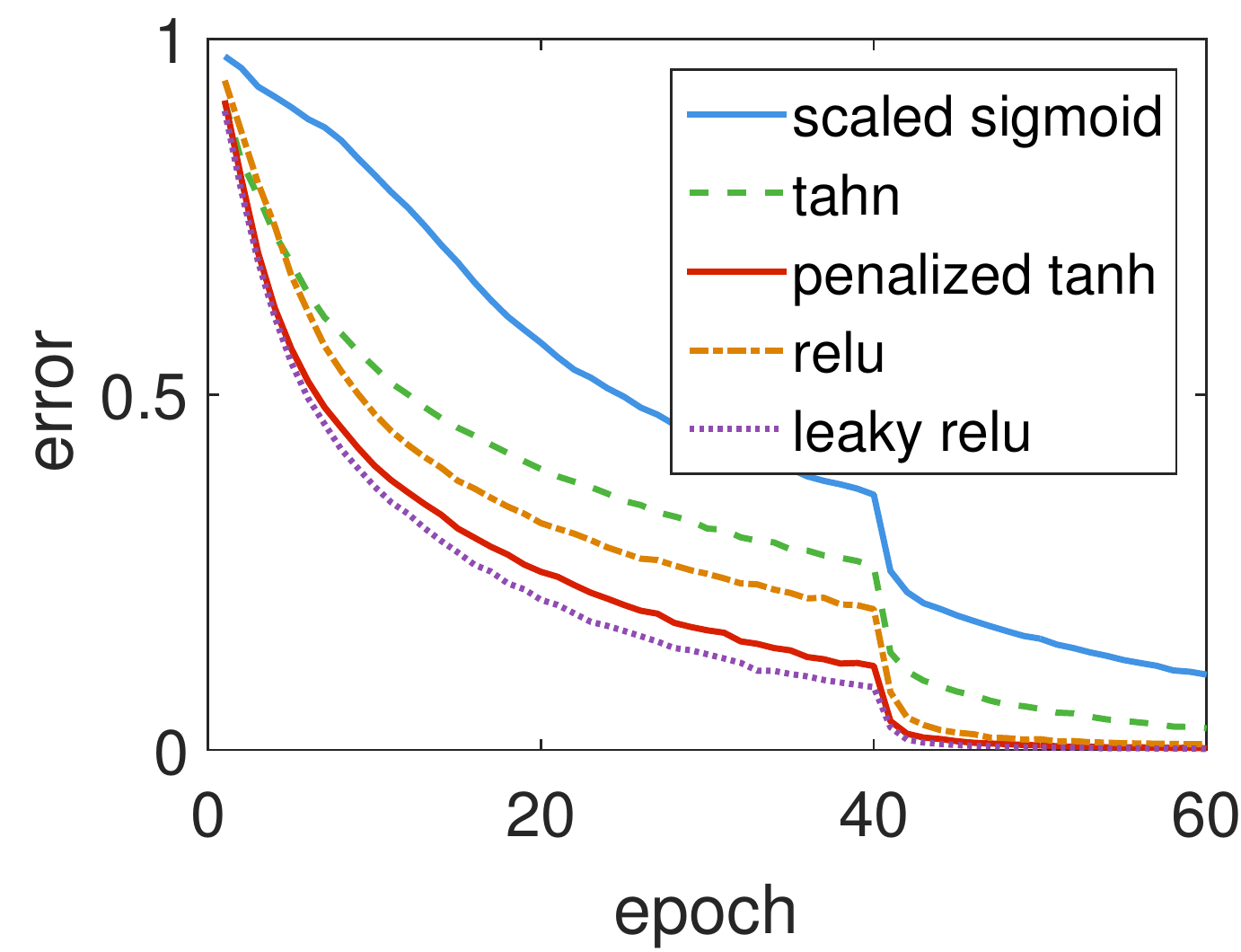}
  \caption{Train}
  \end{subfigure}%
  \begin{subfigure}[b]{0.5\textwidth}
  \includegraphics[width=\textwidth]{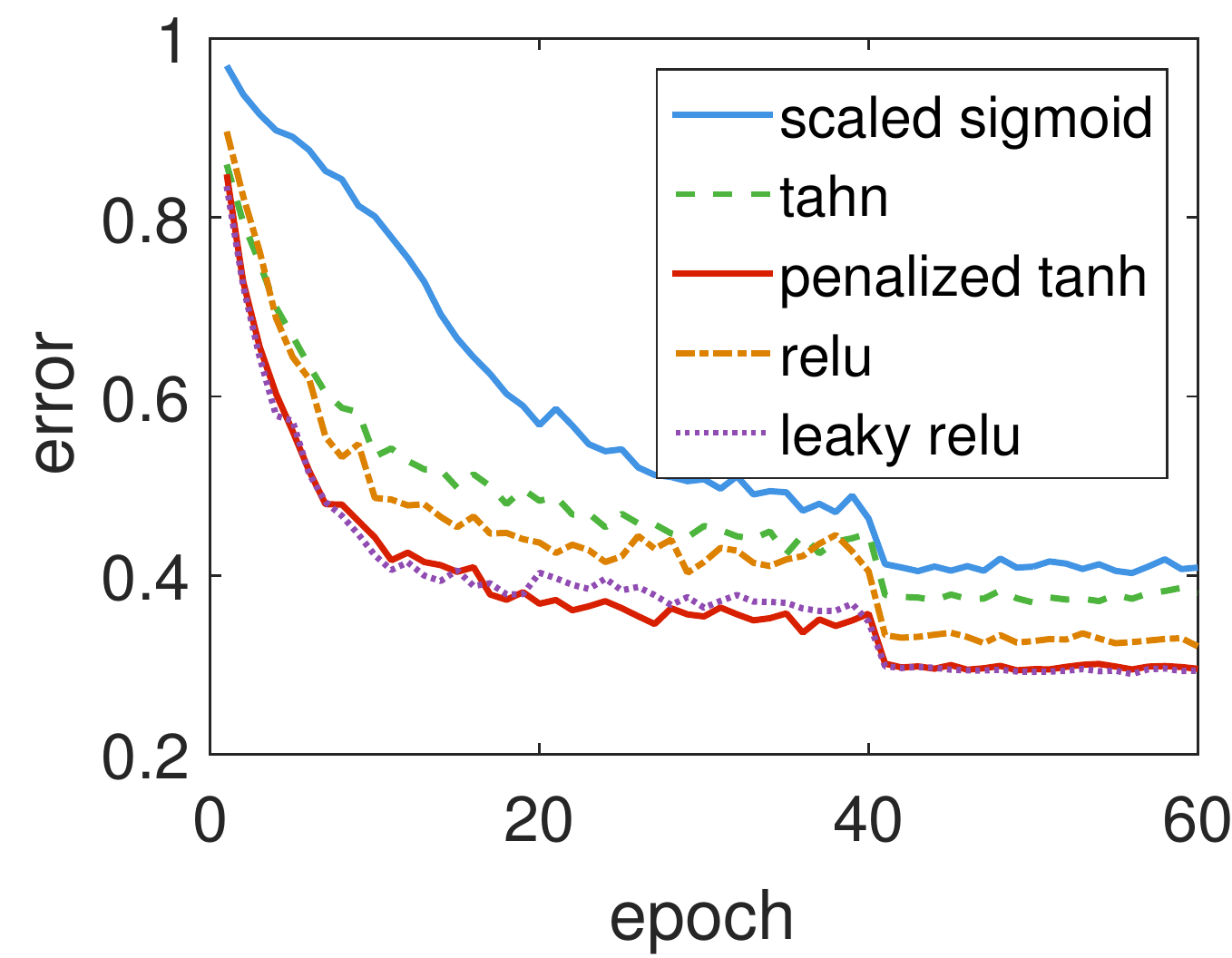}
  \caption{Test}
  \end{subfigure}
  \caption{Error versus epoch for inception network with various activation function on CIFAR-100. }
  \label{fig:conv}
\end{figure}

We evaluated the proposed two activation functions on a 33-layers inception
network without batch normalization \citep{ioffe2015batch}. We used the CIFAR-100 dataset, which is an image classification dataset with 100 classes. All networks are initalized with \citep{glorot2010understanding}.

The converge results are shown in Figure~\ref{fig:conv} and the final training
and test accuracy is reported in Table~\ref{tab:cifar100}.
As expected, ReLU converges faster than tanh but is outperformed
by leaky ReLU.
For the proposed functions, the scaled logistic sigmoid successfully converges
into a local minimal. Its performance is close its cousin tanh, both of them are
saturated and have
have similar shape near 0. On the other hand, the
penalized tanh converges more than 2 times faster than the standard tanh. It
gives almost identical results with its non-saturated version, namely leaky ReLU.

Based on the observations, it seems that the performance of various activation
functions on the inception network mainly depend on the
function shape near 0, namely the values of $f(0)$ and $\partial
f(0)$. It is possibly due to the inputs to the activation function are near 0.

\section{Conclusion \& Future Work}


The result of this paper is two-fold.
We first attempt to  explain and fix the failure of training a deep Sigmoid network, based on the idea of the work \citep{glorot2010understanding}.
A re-scaled  Sigmoid activation is proposed in the paper to make deep Sigmoid network trainable.
The other result of this paper is to investigate the differences in network performances between using saturated activation function and using non-saturated ones.
Our result suggests that when using penalization on negative part, saturation of the activation function is comparable to ReLU and Leaky ReLU.
There are still many open questions requiring further investigation: 1. How to efficiently determine different learning rates for different layers in a very deep neural network? 2. How does the positive part (on $[0, +\infty)$) and the negative part (on $(-\infty, 0]$) of the activation function affect the performance of the network?

\subsubsection*{Acknowledgments}
The authors would like to thank Tianqi Chen, Ian Goodfellow for discussion and NVIDIA for hardware support. 

\bibliography{saturated}
\bibliographystyle{iclr2016_workshop}

\section*{Appendix}
\label{sec:appendix}
\begin{proof}[Proof of Proposition~\ref{prop:var}.]
First note that for each $l=1\ldots, L$, 
\[
\EE{x^\cl{l}} = \EE{W^\cl{l-1}y^\cl{l-1}+b^\cl{l-1}} = \EE{W^\cl{l-1}}\EE{y^\cl{l-1}}+0 = 0,
\]
where the last equality is due to $\EE{W^\cl{l-1}} = 0$. Therefore,
\[
\EE{y^\cl{l}} = \EE{f(x^\cl{l})} = \EE{\alpha x^\cl{l} + \beta \bf{1}} = \alpha \EE{x^\cl{l}} + \beta\bf{1}  =\beta\bf{1}
\]

To compute the variance of $y^\cl{l}$, by definition
\[
\Var{y^\cl{l}} = \Var{\alpha x^\cl{l} + \beta\bf{1}} = \alpha^2\Var{x^\cl{l}}.
\]
Note that  $y^{\cl{l-1}}$ only depends on the weights $W^{0}, \ldots, W^\cl{l-2}$, thus is independent to $W^\cl{l-1}$.
plugging the definition of $x^\cl{l}$,
\begin{align*}
\Var{x^\cl{l}} & = \Var{W^\cl{l-1}y^\cl{l-1}} \\
& = \EE{W^\cl{l-1}y^\cl{l-1}y^{\cl{l-1}\top}W^{\cl{l-1}\top}} - \EE{W^\cl{l-1}y^\cl{l-1}} \EE{W^\cl{l-1}y^\cl{l-1}}^\top \\
& = \EE{W^\cl{l-1}y^\cl{l-1}y^{\cl{l-1}\top}W^{\cl{l-1}\top}} - \EE{W^\cl{l-1}}\EE{y^\cl{l-1}} \EE{y^\cl{l-1}}^\top\EE{W^\cl{l-1}}^\top \\
& = \EE{W^\cl{l-1}y^\cl{l-1}y^{\cl{l-1}\top}W^{\cl{l-1}\top}},
\end{align*}
where the third equality is due to the independence of $y^\cl{l-1}$ and $W^\cl{l-1}$, and the last equality is due to $\EE{W^\cl{l-1}} = 0$.

We will prove that the covariance matrix of $y^\cl{l}$ is $c_lI$ for some constant $c_l$, $l=1,\ldots,L$, using mathematical induction. 
For the base case $y^\cl{0}$ being the input, by assumption the claim holds.
Now assume that the claim holds for $l-1$, i.e. $\Var{y^\cl{l-1}} = c_{l-1}I$, we would like to prove that $\Var{y^\cl{l}} = c_lI$ for some constant $c_l$. To simplify the proof, we would use Lemma \ref{lem:algCal} as follows.
\begin{lemma}
\label{lem:algCal}
Given that $\Var{y} = C I_n$ and $y$ is independent to $W \in \real^{m\times n}$, if the elements of $W$ are mutually independent with common variance $\sigma^2$, then 
\[\EE{Wyy^\top W^\top} = (Cn\sigma^2+ \sigma^2\|\EE{y}\|_2^2)I_m.\]
\end{lemma}

Thus, 
\begin{align*}
\Var{x^\cl{l}} & = \EE{W^\cl{l-1}y^\cl{l-1}y^{\cl{l-1}\top}W^{\cl{l-1}\top}} \\
& = (c_{l-1}n_{l-1}\sigma_{l-1}^2+ \sigma_{l-1}^2\|\EE{y}\|_2^2)I_{n_l} \\
& = n_{l-1}(c_{l-1}+\beta^2)\sigma_{l-1}^2I_{n_l}.
\end{align*}
Picking $c_l =  n_{l-1}(c_{l-1}+\beta^2)\sigma_{l-1}^2$ and the claim holds. Therefore,
\[
\Var{y^\cl{l}} = \alpha^2n_{l-1}(c_{l-1}+\beta^2)\sigma_{l-1}^2I_{n_l} = \alpha^2n_{l-1}\sigma_{l-1}^2\left( \Var{y^\cl{l-1}} +\beta^2 I_{n_l}\right). 
\]

Next we consider the computation of $\Var{\frac{\partial\cost}{\partial y^\cl{l-1}}}$.
Note that 
\[
\frac{\partial\cost}{\partial y^\cl{l-1}} = \frac{\partial y^\cl{l}}{\partial y^\cl{l-1}}\frac{\partial\cost}{\partial y^\cl{l}}.
\]
By definition,
\begin{align*}
\frac{\partial y^\cl{l}}{\partial y^\cl{l-1}} = \frac{\partial x^\cl{l}}{\partial y^\cl{l-1}} \frac{\partial y^\cl{l}}{\partial x^\cl{l}}  = \alpha W^\cl{l-1} 
\end{align*}
Thus, $\frac{\partial y^\cl{l}}{\partial y^\cl{l-1}}$ is independent to $\frac{\partial\cost}{\partial y^\cl{l}}$ which only depends on $W^\cl{t}$ for $t = l,\ldots, L$.
Therefore, given that $\EE{\frac{\partial \cost}{\partial y^\cl{L}}} = 0$,
\begin{align*}
\EE{\frac{\partial \cost}{\partial y^\cl{l-1}}} & = \EE{\alpha W^\cl{l-1}\frac{\partial\cost}{\partial y^\cl{l}}} \\
	& = \alpha \EE{W^\cl{l-1}} \EE{ \frac{\partial\cost}{\partial y^\cl{l}}} \\
	& = 0.
\end{align*}
We can now compute the variance of $\frac{\partial\cost}{\partial y^\cl{l-1}}$.
\begin{align*}
\Var{\frac{\partial\cost}{\partial y^\cl{l-1}}} & = \Var{\alpha W^\cl{l-1} \frac{\partial\cost}{\partial y^\cl{l}}} \\
& = \alpha^2 \Var{W^\cl{l-1} \frac{\partial\cost}{\partial y^\cl{l}}} \\
& = \alpha^2 \EE{W^\cl{l-1} \frac{\partial\cost}{\partial y^\cl{l}} \frac{\partial\cost}{\partial y^\cl{l}}^\top W^{\cl{l-1}\top}}
\end{align*}
Note that by assumptions, $\frac{\partial \cost}{\partial y^\cl{L}}$ is independent to $W^\cl{t}$, $t = 1,\ldots,L-1$, with $\Var{\frac{\partial \cost}{\partial y^\cl{L}}} = d_LI$ for some constant $d_L$. Also $\EE{\frac{\partial \cost}{\partial y^\cl{l}}} = 0$.
Similar to the first part of the proof, given that $\frac{\partial \cost}{\partial y^\cl{L}}$ is independent to $W^\cl{t}$, $t = 1,\ldots,L-1$, with $\Var{\frac{\partial \cost}{\partial y^\cl{L}}} = d_LI$ for some constant $d_L$,
one can prove that  $\Var{\frac{\partial \cost}{\partial y^\cl{l}}} = d_l I$ for some constant $d_l$, and thus 
\[
\Var{\frac{\partial \cost}{\partial y^\cl{l-1}}} = \alpha^2(d_ln_l+ \|\EE{\frac{\partial \cost}{\partial y^\cl{l}}}\|_2^2)\sigma_{l-1}^2I_{n_{l-1}} = d_ln_l\sigma_{l-1}^2I_{n_{l-1}} = \alpha^2n_l\sigma_{l-1}^2\Var{\frac{\partial \cost}{\partial y^\cl{l}}}.
\]
\end{proof}

\begin{proof}[Proof of Lemma \ref{lem:algCal}]
\begin{align*}
\EE{Wyy^\top W^\top} & = \EE{Wyy^\top W^\top - W\EE{y}\EE{y}^\top  W^\top} + \EE{W\EE{y}\EE{y}^\top W^\top} \\
	& = \EE{W\Var{y}W^\top} +  \EE{W\EE{y}\EE{y}^\top  W^\top} \\
	& = C\EE{WW^\top}+  \EE{W\EE{y}\EE{y}^\top  W^\top}.
\end{align*}
Here the last equality is due to $\Var{y} = cI_n$. For the first term, consider its $(i,j)$th element
\[
\EE{W_{i:}W_{j:}^\top} = \begin{cases}
0 & \text{if } i\neq j \\
n\sigma^2 & \text{if } i=j
\end{cases},
\]
where $W_{i:}$ is the $i$th row of $W$.
Similarly for the second term $\EE{W\EE{y}\EE{y}^\top  W^\top}$, its $(i,j)$th element 
\[
\EE{W_{i:}\EE{y}\EE{y}^\top W_{j:}^\top} = \EE{y}^\top\EE{W_{i:}^\top W_{j:}}\EE{y}= \begin{cases}
0 & \text{if } i\neq j \\
\sigma^2 \|\EE{y}\|_2^2 & \text{if } i=j
\end{cases}.
\]
 Therefore,
\[\EE{Wyy^\top W^\top} = (Cn\sigma^2+ \sigma^2\|\EE{y}\|_2^2)I_m.\]
\end{proof}

\end{document}